\documentclass{bmvc2k}
\usepackage{times}
\usepackage{epsfig}
\usepackage{graphicx}
\usepackage{amsmath}
\usepackage{amssymb}
\usepackage{amsthm}
\usepackage{booktabs}
\usepackage{bm}
\usepackage{xcolor}

\usepackage{listings}

\definecolor{codegreen}{rgb}{0,0.6,0}
\definecolor{codegray}{rgb}{0.5,0.5,0.5}
\definecolor{codepurple}{rgb}{0.58,0,0.82}
\definecolor{backcolour}{rgb}{0.95,0.95,0.92}

\def\Tabref#1{Table~\ref{#1}}

\def\Figref#1{Figure~\ref{#1}}

\def\eqref#1{equation~\ref{#1}}
\def\Eqref#1{Equation~\ref{#1}}

\def\1{\bm{1}}

\def\vc{{\bm{c}}}

\def\vf{{\bm{f}}}

\def\mF{{\bm{F}}}

\def\mK{{\bm{K}}}
\def\mL{{\bm{L}}}

\def\mP{{\bm{P}}}
\def\mQ{{\bm{Q}}}

\DeclareMathAlphabet{\mathsfit}{\encodingdefault}{\sfdefault}{m}{sl}
\SetMathAlphabet{\mathsfit}{bold}{\encodingdefault}{\sfdefault}{bx}{n}

\DeclareMathOperator*{\argmin}{arg\,min}

\newtheorem{lemma}{Lemma}

\title{DISCO: accurate Discrete Scale Convolutions}

\addauthor{Ivan Sosnovik}{i.sosnovik@uva.nl}{1}
\addauthor{Artem Moskalev}{a.moskalev@uva.nl}{1}
\addauthor{Arnold Smeulders}{a.w.m.smeulders@uva.nl}{1}

\addinstitution{
 UvA-Bosch Delta Lab\\
 University of Amsterdam\\
 Netherlands
}

\runninghead{Sosnovik, Moskalev, Smeulders}{DISCO}

\def\etal{\emph{et al}\bmvaOneDot}

\begin{document}

\maketitle

\begin{abstract}
Scale is often seen as a given, disturbing factor in many vision tasks. When doing so it is one of the factors why we need more data during learning. In recent work scale equivariance was added to convolutional neural networks. It was shown to be effective for a range of tasks. We aim for accurate scale-equivariant convolutional neural networks (SE-CNNs) applicable for problems where high granularity of scale and small kernel sizes are required. Current SE-CNNs rely on weight sharing and kernel rescaling, the latter of which is accurate for integer scales only. To reach accurate scale equivariance, we derive general constraints under which scale-convolution remains equivariant to discrete rescaling. We find the exact solution for all cases where it exists, and compute the approximation for the rest. The discrete scale-convolution pays off, as demonstrated in a new state-of-the-art classification on MNIST-scale and on STL-10 in the supervised learning setting. With the same SE scheme, we also improve the computational effort of a scale-equivariant Siamese tracker on OTB-13.
\end{abstract}

\section{Introduction}
\label{sec:intro}

Scale is a natural attribute of every object, as basic property as location and appearance. And hence it is a factor in almost every task in computer vision. In image classification, global scale invariance plays an important role in achieving accurate results \cite{kanazawa2014locally}. In image segmentation, scale equivariance is important as the output map should scale proportionally to the input \cite{anderson1988adaptive}. And in object detection or object tracking, it is important to be scale-agnostic \cite{ren2015faster}, which implies the availability of both scale invariance as well as scale equivariance as the property of the method. Where scale invariance or equivariance is usually left as a property to learn in the training of these computer vision methods by providing a good variety of examples \cite{lin2017feature}, we aim for accurate scale analysis for the purpose of needing less data to learn from.

Scale of the object can be derived externally from the size of its silhouette, e.g \cite{wu2018size}, or internally from the scale of its details, e.g \cite{chang2009analysis}. External scale estimation requires the full object to be visible. It will easily fail when the object is occluded and/or when the object is amidst a cluttered background, for example for people in a crowd \cite{smeulders2013visual}, when proper detection is hard. In contrast, internal scale estimation is build on the scale of common details \cite{schneiderman2004object}, for example deriving the scale of a person from the scale of a sweater or a face. Where internal scale has better chances of being reliable, it poses heavier demands on the accuracy of assessment than external scale estimation. We focus on improvement of the accuracy of internal scale analysis.

\begin{figure}[t]
  \centering
    \includegraphics[width=0.9\linewidth]{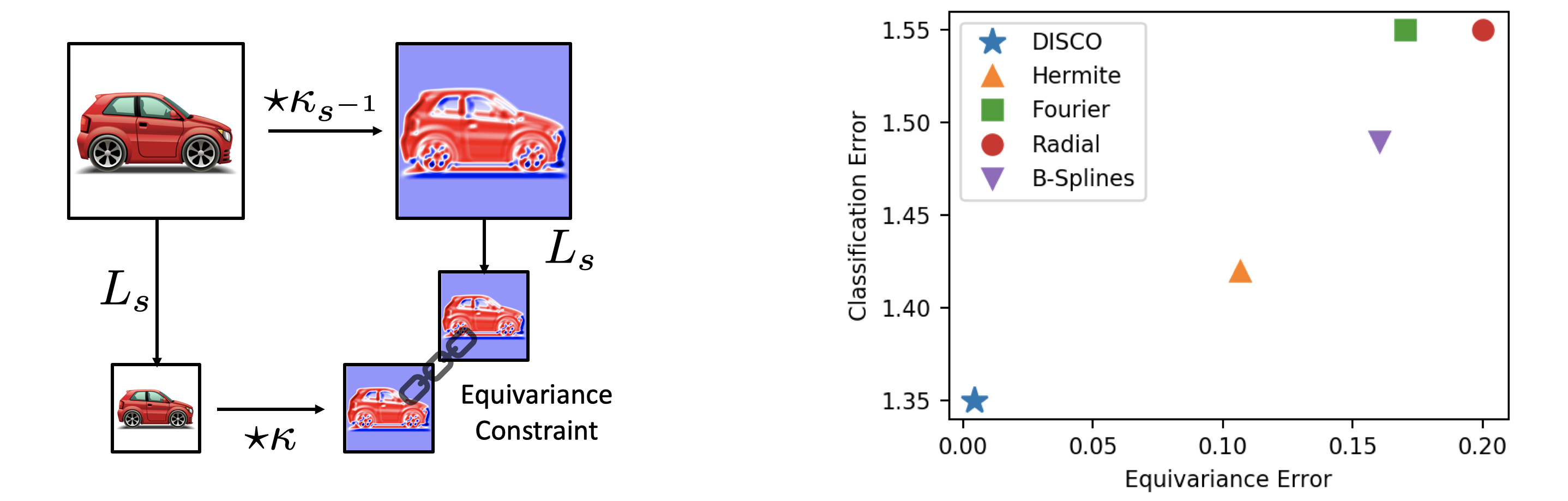}
  \caption{Left: the necessary constraint for scale-equivariance. When it is not satisfied an \textit{equivariance error} appears. Right: Equivariance error vs. Classification error for scale-equivariant models on MNIST-scale. DISCO achieves the lowest equivariance error and this leads to the best classification accuracy. Alongside DISCO, we test SESN models  with Hermite \cite{sosnovik2019scale}, Fourier \cite{zhu2019scale}, Radial \cite{ghosh2019scale} and B-Spline \cite{bekkers2019b} bases.}
  \label{fig:fig1intro}
\end{figure}

We focus on accurate scale analysis on the generally applicable scale-equivariant convolutional neural networks \cite{worrall2019deep,bekkers2019b,sosnovik2019scale}. A scale-equivariant network extends the equivariant property of conventional convolutions to the scale-translation group. It is achieved by rescaling the kernel basis and sharing weights between scales. While the weight sharing is defined by the structure of the group \cite{cohen2016group}, the proper way to rescale kernels is an open problem. In \cite{bekkers2019b,sosnovik2019scale}, the authors propose to rescale kernels in the continuous domain to project them later on a pixel grid. This permits the use of arbitrary scales, which is important to many application problems, but the  procedure may cause a significant equivariance error \cite{sosnovik2019scale}. Therefore, Worrall and Welling \cite{worrall2019deep} model rescaling as a dilation, which guarantees a low equivariance error at the expense of permitting only integer scale factors. Due to the continuous nature of observed scale in segmentation, tracking or classification alike, integer scale factors may not cover the  range of variations in the best possible way. 

In the paper, we show how the equivariance error affects the performance of SE-CNNs. We make the following contributions: 

\begin{itemize}
    \item From first principles we derive the best kernels, which minimize the equivariance error.
    \item We find the conditions when the solution exists and find a good approximation when it does not exist.
    \item We demonstrate that an SE-CNN with the proposed kernels outperforms recent SE-CNNs in classification and tracking in both accuracy and compute time. We set new state-of-the-art results on MNIST-scale and STL-10.
\end{itemize}

The proposed approach contains \cite{worrall2019deep} as a special case. Moreover, the proposed kernels can't be derived from \cite{sosnovik2019scale} and vice versa. The union of our approach and the approach presented in \cite{sosnovik2019scale} covers the whole set of possible SE-CNNs for a finite set of scales.
\section{Related Work}
\label{sec:related_work}

\vspace{-3mm}
\paragraph{Group Equivariant Networks.}
In recent years, various works on group-equivariant convolution neural networks have appeared. In majority, they consider the roto-translation group in 2D \cite{cohen2016group,cohen2016steerable,hoogeboom2018hexaconv,worrall2017harmonic,weiler2019general,weiler2018learning}, 
the roto-translation group in 3D \cite{worrall2018cubenet,kondor2018n,thomas2018tensor,cohen2017convolutional,weiler20183d}, the compactified rotation-scaling group in 2D \cite{henriques2017warped} and the rotation group 3D 
\cite{cohen2017convolutional,esteves2018learning,cohen2019gauge}. In \cite{cohen2018general,kondor2018generalization,lang2020wigner} the authors demonstrate how to build convolution networks equivariant to arbitrary compact groups. All these papers cover group-equivariant networks for compact groups. In this paper, we focus the scale-translation group which is an example of a non-compact group.

\vspace{-3mm}
\paragraph{Discrete Operators.}
Minimization of the discrepancies between the theoretical properties of continuous models and their discrete realizations has been studied for a variety of computer vision tasks. Lindeberg \cite{lindeberg1990scale,lindeberg2013scale} proposed a method for building a scale-space for discrete signals. The approach relied on the connection between the discretized version of the diffusion equation and the structure of images. While this method considered the scale symmetry of images and significantly improved computer vision models in the pre-deep-learning era, it is not directly applicable to our case of scale-equivariant convolutional networks. 

In \cite{diaconu2019learning}, Diaconu and Worrall demonstrate how to construct rotation-equivariant CNNs on the pixel grid for arbitrary rotations. The authors propose to learn the kernels which minimize the equivariance error of rotation-equivariant convolutional layers. The method relies on the properties of the rotation group and cannot be generalized to the scale-translation group. In this paper, we show how to minimize the equivariance error for scale-convolution without the use of extensive learning.

\vspace{-3mm}
\paragraph{Scale-Equivariant CNNs.}
An early work of \cite{kanazawa2014locally} introduced SI-ConvNet, a model where the input image is rescaled into a multi-scale pyramid. Alternatively, Xu \etal \cite{xu2014scale} proposed SiCNN, where a multi-scale representation is built from rescaling the network filters. While these modified convolutional networks significantly improve image classification, they require run-time interpolation. As a result they are several orders slower than standard CNNs. 

In \cite{sosnovik2019scale,bekkers2019b,zhu2019scale} the authors propose to parameterize the filters by a trainable linear combination of a pre-calculated, fixed multi-scale basis. Such a basis is defined in the continuous scale domain and projected on a pixel grid for the set of scale factors. The models do not involve interpolation during training nor inference. As a consequence, they operate within reasonable time. The continuous nature of the bases allows for the use of arbitrary scale factors, but it suffers from a reduced accuracy as the projection on the discrete grid causes an equivariance error.

Worral and Welling \cite{worrall2019deep} propose to model filter rescaling by dilation. This solves the equivariance error of the previous method at the price of permitting only integer scale factors. That makes the method less suited for object tracking, depth analysis and fine-grained image classification, where subtle changes in the image scale are important in the performance. Our approach combines the best of the both worlds as it guarantees a low equivariance error for arbitrary scale factors.

\vspace{-3mm}
\paragraph{Accurate Scale Analysis.}
Approaches based on feature pyramids are applied in many tasks \cite{han2017deep,lin2017feature,wang2020scale,qiao2020detectors}. Their implementation require a significant specialisation of the network architecture. Models based on direct scale regression \cite{ren2015faster,li2019siamrpn++,chen2020siamese} have proved to be accurate in scale analysis, but they rely on a complicated training procedure. Scale-equivariant networks require only a drop-in replacement of the standard convolutions by scale-convolutions, while keeping the training procedure unchanged \cite{worrall2019deep,bekkers2019b,sosnovik2019scale,Sosnovik_2021_WACV}. We appreciate the universal applicability of scale-equivariant networks. We focus on this particular use in our implementation while the method we set out in this paper will also apply to other ways of using scale in computer vision.

Existing models for scale-equivariant networks bring computational overhead, which significantly slows down the training and the inference. In this paper, we present scale-equivariant models which allow for the accurate analysis of scale with a minimum computational overhead while retaining the advantage of being an easy replacement of convolutional layers to improve.
\section{Method}
\label{sec:method}
\vspace{-3mm}
\paragraph{Equivariance.} 
A mapping $g$ is equivariant under a transformation $L$ if and only if there exists $L'$ such that $g \circ L  = L' \circ g$. If the mapping $L'$ is identity, then $g$ is invariant under transformation $L$. 

\vspace{-3mm}
\paragraph{Scale Transformations.}
Given a function $f: \mathbb{R}\rightarrow \mathbb{R}$ its scale transformation $L_s$ is defined by
\begin{equation}
    \label{eq:def_scale_transformation}
    L_s[f](t) = f(s^{-1}t), \quad \forall s >0
\end{equation}
We refer to cases with $s>1$ as up-scalings and to cases with $s < 1$ as down-scalings, where $L_{1/2}[f]$ stands for a function down-scaled by a factor of 2.

\vspace{-3mm}
\paragraph{The scale-translation group.}
We are interested in equivariance under the scale-translation group $H$ and its subgroups. It consists of the translations $t$ and scale transformations $s$ which preserve the position of the center. $H=\{(s, t)\}=S \rtimes T$ is a semi-direct product of a multiplicative group $S = (\mathbb{R}^+, +)$ and an additive group $T = (\mathbb{R}, +)$. For the multiplication of its elements we have $(s_2, t_2) \cdot (s_1, t_1) = (s_1s_2, s_2t_1 + t_2)$. Scale transformation of a function defined on group $H$ consists of a scale transformation of its spatial part as it is defined in the \Eqref{eq:def_scale_transformation} and a corresponding multiplicative transformation of its scale part. In other words
\begin{equation}
    \label{eq:def_scale_h}
    L_{\hat{s}}[f](s, t) = f(s\hat{s}^{-1}, \hat{s}^{-1}t)
\end{equation}

\subsection{Scale-Convolution}
A scale-convolution of $f$ and a kernel $\kappa$ both defined on scale $s$ and translation $t$ is given by: \cite{sosnovik2019scale}:
\begin{equation}
    \label{eq:sc_conv_def}
    [f \star_H \kappa](s, t) = \sum_{s'} [f(s', \cdot) \star \kappa_s(s^{-1}s', \cdot)](\cdot, t)
\end{equation}
where $\kappa_s$ stands for an $s$-times up-scaled kernel $\kappa$, $\star$ and $\star_H$ are convolution and scale-convolution. The exact way the up-scaling is performed depends on how the down-scaling of the input signal works. 

Scale-convolution is equivariant to transformations $L_{\hat{s}}$ from the group $H$, therefore the following holds true by definition:
\begin{equation}
    \label{eq:sc_conv_commutes}
    [L_{\hat{s}}[f] \star_H \kappa] = L_{\hat{s}}[f \star_H \kappa]
\end{equation}

Expanding the left-hand side of this relation by using \Eqref{eq:sc_conv_def}, choosing $s=1$ and replacing $s' \rightarrow s' \hat{s}$ we find:
\begin{equation}
    \label{eq:kernel_constraint_lhs}
    [L_{\hat{s}}[f] \star_H \kappa](s, t) = \sum_{s'} [L_{\hat{s}}[f(s', \cdot)] \star \kappa(\hat{s}s', \cdot)](\cdot, t)
\end{equation}
For the right-hand side we have:
\begin{equation}
    \label{eq:kernel_constraint_rhs}
   L_{\hat{s}}[f \star_H \kappa](s, t) = \sum_{s'} L_{\hat{s}}[f(s', \cdot) \star \kappa_{\hat{s}^{-1}}(\hat{s}s', \cdot)](\cdot, t)
\end{equation}
Equating the two sides and choosing $f$ to be zero on all scales but $s=1$, we obtain the equivariance constraint for the kernels 
\begin{equation}
    \label{eq:equi_constraint}
    L_s[f] \star \kappa = L_s[f \star \kappa_{s^{-1}}], \quad \forall f, s
\end{equation}
We have found that \textit{the mapping defined by \Eqref{eq:sc_conv_def} is scale-equivariant only if a kernel and its up-scaled versions satisfy \Eqref{eq:equi_constraint}}. Thus, it proves to be the necessary condition for scale-equivariant convolutions. In \cite{sosnovik2019scale,bekkers2019b,zhu2019scale} the opposite, sufficient condition was proved. As a whole it defines the relation between scale convolution and the constraints of its kernels.

\subsection{Exact Solution}
In the continuous domain, convolution is defined as an integral over the spatial coordinates. \cite{sosnovik2019scale,bekkers2019b,zhu2019scale} derives a solution for \Eqref{eq:equi_constraint}:
\begin{equation}
    \label{eq:continuous_solution}
    \kappa_s(t) = s^{-1}\kappa(s^{-1}t)
\end{equation}
However, when such kernels are calculated and projected on the pixel grid, a discrepancy between the left-hand side and the right-hand side of \Eqref{eq:equi_constraint} will appear. We refer to such inequality as the \textit{equivariance error}.

We aim at directly solving \Eqref{eq:equi_constraint} in the discrete domain. In general, for discrete signals down-scaling is a non-invertible operation. Thus $L_s$ is well-defined only for $s<1$. We start by solving \Eqref{eq:equi_constraint} for 1-dimensional discrete signals. We prove its generalization to the 2-dimensional case in supplementary materials. \Figref{fig:dilation} illustrates the approach. 

Let us consider a discrete signal $f$ represented as a vector $\vf$ of length $N_\text{in}$. It is down-scaled to length $N_\text{out} < N_\text{in}$ by $L_s$, which is represented as a rectangular interpolation matrix $\mL$ of size $N_\text{out} \times N_\text{in}$. A convolution with a kernel $\kappa$ is represented as a multiplication with a matrix $\mK$ of size $N_\text{out}\times N_\text{out}$, and with a kernel $\kappa_{s^{-1}}$ written as a matrix $\mK_{s^{-1}}$ of size $N_\text{in}\times N_\text{in}$. Then \Eqref{eq:equi_constraint} can be rewritten in matrix form as follows:
\begin{equation}
    \label{eq:equi_constrain_matrix}
    \mK\mL\vf = \mL\mK_{s^{-1}}\vf, \; \forall \vf
    \quad \Longleftrightarrow \quad
    \mK \mL = \mL \mK_{s^{-1}}
\end{equation}

Without loss of generality we assume circular boundary conditions. Then the matrix representations $\mK$ and $\mK_{s^{-1}}$ are both circulant and their eigenvectors are the column-vectors of the Discrete Fourier Transform $\mF$ \cite{bamieh2018discovering,henriques2014fast,chang2000fast}:
\begin{equation}
    \label{eq:kernel_eigend}
    \mK_{s^{-1}}=\mF \text{diag}(\mF\bm{\kappa}_{s^{-1}})\mF^*
\end{equation}
where $\bm{\kappa}_{s^{-1}}$ is a vector representation of $\kappa_{s^{-1}}$ padded with zeros. After substituting \Eqref{eq:kernel_eigend} into \Eqref{eq:equi_constrain_matrix} and multiplying both sides by $\mF$ from the right, we get:
\begin{equation}
    \label{eq:kernel_constraint_last_step}
    \mK \mL \mF = \mL \mF \text{diag}(\mF\bm{\kappa}_{s^{-1}})
\end{equation}
The left-hand side of the equation is obtained from $\mL \mF$ by multiplying it with a diagonal matrix from the right. Thus, each column of the matrix $\mK \mL \mF$ is proportional to the corresponding column of the matrix $\mL \mF$. We prove in supplementary materials that \textit{such a relation is possible if and only if the matrix $\mL$ performs a down-scaling by an integer scale factor}. 

When the requirement is satisfied, the solution with respect to $\bm{\kappa}_{s^{-1}}$ is the dilation of $\bm{\kappa}$ by factor $s$. Such a solution also known as the \textit{\`a trous algorithm} \cite{holschneider1990real}:
\begin{equation}
    \label{eq:exact_sol_1d}
    (\bm{\kappa}_{s^{-1}})_{is} = \sum_i \mF^*_{ij}
    (\mK\mL\mF)_{1j} / (\mL\mF)_{1j} = \bm{\kappa}_i
\end{equation}

\begin{figure*}[t]
  \centering
    \includegraphics[width=0.98\linewidth]{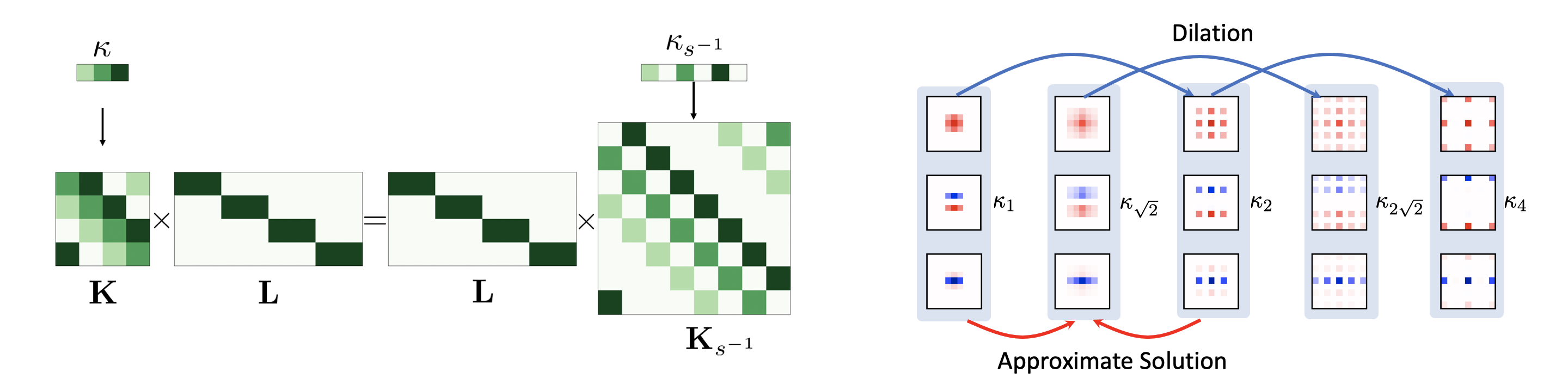}
  \caption{Left: a matrix representation of the 1-dimensional case of the equivariance constraint for $N_{\text{in}}=8$ and $N_{\text{out}}=4$. Right: a multi-scale kernel initialization. $\sqrt{2}$ is the smallest non-integer scale, for which the kernel is approximated by minimizing \Eqref{eq:approx_solution_kernel}, the rest of the kernels can be obtained with dilation.}
  \label{fig:dilation}
\end{figure*}

\subsection{Approximate solution}
Let us consider a scale-convolutional layer. One of its hyper-parameters is the set of scales it operates on. For the cases of non-integer scale factors any kernels will introduce an equivariance error into the network. Thus, it is reasonable to use integer scales as reference points and add intermediate scales to cover the required range of scale factors best. Let us choose a set of scales $\{1, \sqrt{2}, 2, 2\sqrt{2}, 4, 4\sqrt{2}, \dots\}$. The set of corresponding kernels is $\{\kappa_1, \kappa_{\sqrt{2}}, \kappa_2, \kappa_{2\sqrt{2}}, \dots\}$. As the smallest kernel is known, all kernels defined on integer scales can be calculated as its dilated versions. And, when kernel $\kappa_{\sqrt{2}}$ is defined, all intermediate kernels $\kappa_{2\sqrt{2}}, \kappa_{4\sqrt{2}}, \dots$ can be calculated by using dilation as well. Thus, the only kernel yet unknown is kernel $\kappa_{\sqrt{2}}$.

The kernel $\kappa_{\sqrt{2}}$ can be calculated as a minimizer of the equivariance error based on the \Eqref{eq:equi_constraint} as follows:
\begin{equation}
    \label{eq:approx_solution_kernel}
    \kappa_{\sqrt{2}} = 
    \argmin \mathbb{E}_{f} 
    \|L[f] \star \kappa_1 - L[f \star \kappa_{\sqrt{2}}]\|_F^2 
    + \|L[f] \star \kappa_{\sqrt{2}} - L[f \star \kappa_2]\|_F^2
\end{equation}
where $L = L_{1/\sqrt{2}}$ is a down-scaling by a factor $\sqrt{2}$. 

We demonstrate how to calculate approximate solution for the most general case in supplementary materials.

\subsection{Implementation}
To construct scale-equivariant convolution we parametrize the kernels as a linear combination of fixed multi-scale basis. The basis is then fixed and only corresponding coefficients are trained. The coefficients are shared for all scales.

We utilize the standard pixel basis on the smallest integer scale. The bases for the rest of the integer scales are computed as a dilation. The basis on the smallest non-integer scale is approximated by applying gradient descent to \Eqref{eq:approx_solution_kernel}. We note that it takes negligible time to compute all of the basis functions before training. See supplementary materials for more details. We refer to scale-convolutions with the proposed bases as Discrete Scale Convolutions or shortly DISCO. As DISCO kernels are sparse, they allow for lower computational complexity. 

\begin{table*}[t]
    \begin{center}
    \begin{tabular}{|l|c|cc|c|c|}
        \hline
        Model      &Basis & MNIST & MNIST+  & Equi. error& \# Params.\\ 
        \hline\hline
        CNN             & - & $2.02 \pm 0.07$  &$1.60 \pm 0.09$   & - & 495 K  \\
        SiCNN           & - & $2.02 \pm 0.14$  &$1.59 \pm 0.03$   & - & 497 K    \\
        SI-ConvNet      & - & $1.82 \pm 0.11$  &$1.59 \pm 0.10$  & - & 495 K\\
        SEVF            & - & $2.12 \pm 0.13$  &$1.81 \pm 0.09$  & - & 475 K\\
        DSS             & Dilation & $1.97 \pm 0.08$  &$1.57 \pm 0.09$  & 0.0 & 494 K\\
        SS-CNN          & Radial  & $1.84 \pm 0.10$  &$1.76 \pm 0.07$  & - & 494 K \\
        \hline
        SESN &Hermite   & $1.68 \pm 0.06$  &$1.42 \pm 0.07$  & 0.107 & 495 K  \\
        SESN &B-Spline  & $1.74 \pm 0.08$  &$1.49 \pm 0.05$  & 0.163 & 495 K  \\
        SESN &Fourier & $1.88 \pm 0.07$  &$1.55 \pm 0.07$  & 0.170 & 495 K  \\
        SESN &Radial  & $1.74 \pm 0.07$  &$1.55 \pm 0.10$  & 0.200 & 495 K  \\
        \hline
        DISCO & Discrete & $\mathbf{1.52 \pm 0.06}$ &$\mathbf{1.35 \pm 0.05}$  & 0.004 & 495 K  \\
        \hline  

    \end{tabular}
    \end{center}
    \caption{The classification error of various methods on the MNIST-scale dataset, lower is better. We test both the regime with and without data augmentation, where scaling data augmentation is denoted by ``$+$''. All results are reported as mean $\pm$ std over 6 different, fixed realizations of the dataset. The best results are \textbf{bold}.}
    \label{tab:mnist_scale_results}
\end{table*}

\section{Experiments}
\label{sec:experiments}
\subsection{Equivariance Error}

To quantitatively evaluate the equivariance error of DISCO versus other methods for scale-convolution \cite{sosnovik2019scale,zhu2019scale,bekkers2019b}, we follow the approach proposed in \cite{sosnovik2019scale}. In particular, we randomly sample images from the MNIST-Scale dataset \cite{sosnovik2019scale} and pass in through the scale-convolution layer. Then, the equivariance error is calculated as follows:
\begin{equation}
    \label{eq:equiv_error}
    \Delta = \sum_s 
    \| L_{s} \Phi(f) - \Phi(L_{s}f) \|_2^2 
    /
    \| L_{s} \Phi(f) \|_2^{2}
\end{equation}
where $\Phi$ is scale-convolution with weights initialized randomly.

The equivariance error for each model is reported in \Tabref{tab:mnist_scale_results} and in \Figref{fig:fig1intro}. Note that we can not directly compare against \cite{worrall2019deep} as it only permits integer scale factors. As can be seen, there exists a correlation between an equivariance error and classification accuracy. DISCO model attains the lowest equivariance error.

\subsection{Image Classification}
We conduct several experiments to compare various methods for scale analysis in image classification.
Alongside DISCO, we test SI-ConvNet \cite{kanazawa2014locally}, SS-CNN \cite{ghosh2019scale}, SiCNN \cite{xu2014scale}, SEVF \cite{marcos2018scale}, DSS \cite{worrall2019deep} and SESN \cite{sosnovik2019scale}. By relying on the code provided by the authors we additionally reimplement SESN models with other bases such as B-Splines \cite{bekkers2019b}, Fourier-Bessel Functions \cite{zhu2019scale} and Log-Radial Harmonics \cite{ghosh2019scale,naderi2020scale}. 

\paragraph{MNIST-scale.} 
Following \cite{sosnovik2019scale} we conduct experiments on the MNIST-scale dataset. The dataset consists of 6 splits, each of which contains 10,000 images for training, 2,000 for validation and 50,000 for testing. Each image is a randomly rescaled version of the original from MNIST \cite{lecun1998gradient}. The scaling factors are uniformly sampled from the range of $0.3-1.0$.

As a baseline model we use the SESN model, which holds the state-of-the-art result on this dataset. Both SESN and DISCO use the same set of scales in scale convolutions: $\{1, 2^{1/3}, 2^{2/3}, 2\}$ and are trained in exactly the same way. As can be seen from \Tabref{tab:mnist_scale_results}, our DISCO model outperforms other scale equivariant networks in accuracy and equivariance error and sets a new state-of-the-art result.

\begin{table}[t]
    \begin{center}
    \begin{tabular}{|l|cccccc|c|}
        \hline
        Model & WRN & SiCNN & SI-ConvNet & DSS & SS-CNN & SESN & DISCO \\
        \hline\hline
        Basis & - & - & - & Dilation & Radial & Hermite & Discrete \\
        Time, s & 10 & 110 & 55 & 40 & 15 & 165 & 50 \\
        Error & $11.48$ & $11.62$ & $12.48$ & $11.28$  & $25.47$ & $8.51$ & $\mathbf{8.07}$ \\
        \hline
    \end{tabular}
    \end{center}
    \caption{The classification error on STL-10. The best results are in \textbf{bold}. The average compute time per epoch is reported in seconds. DISCO sets a new state-of-the-art result in the supervised learning setting.}
    \label{tab:stl_results}
\end{table}

\begin{table}[t]
    \begin{center}
    \begin{tabular}{|cc|}
    \hline
    Equi. Error & STL-10 Error \\
    \hline\hline
    $0.240$ & $8.63$\\
    $0.082$ & $8.25$\\
    $\mathbf{0.003}$ & $\mathbf{8.07}$\\
    \hline
    \end{tabular}
    \end{center}
    \caption{Classification accuracy on STL-10 and the equivariance error for the DISCO model with different filters. The first and the second rows correspond to the cases when the basis for the intermediate scale is not optimized.}
    \label{tab:equi_error}
\end{table}

\paragraph{STL-10.} To demonstrate how accurate scale equivariance helps when the training data is limited, we conduct experiments on the STL-10 \cite{coates2011analysis} dataset. This dataset consists of just 8,000 training and 5,000 testing images, divided into 10 classes. Each image has a resolution of $96\times96$ pixels. 

As a baseline we use WideResNet \cite{Zagoruyko2016WRN} with 16 layers and a widening factor of 8. Scale-equivariant models are constructed according to \cite{sosnovik2019scale}. All models have the same number of parameters, the same set of scales $\{1, \sqrt{2}, 2\}$ and are trained for the same number of steps. For testing the disco model we use exactly the same setup as described by the authors of \cite{sosnovik2019scale}. All the models are trained on NVidia GTX 1080 Ti.

The models are trained for 1000 epochs using the SGD optimizer with a Nesterov momentum of $0.9$ and a weight decay of $5\cdot10^{-4}$. For DISCO, we increase the weight decay to $1\cdot10^{-4}$. Tuning weight decay for the other models did not bring any improvement. The learning rate is set to $0.1$ at the start and decreased by a factor of $0.2$ after the epochs 300, 400, 600 and 800. The batch size is set to 128. During training, we additionally augment the dataset with random crops, horizontal flips and cutout \cite{devries2017cutout}.

As can be seen from \Tabref{tab:stl_results}, the proposed DISCO model outperforms the other scale-equivariant networks and sets a new state-of-the-art result in the supervised learning setting. Moreover, DISCO is more than 3 times faster than the second-best SESN-model.

We additionally check how accuracy degrades if the basis for the scale of $\sqrt{2}$ is not correctly calculated. While the optimal basis is a minimizer of \Eqref{eq:approx_solution_kernel}, it is possible to stop the stop optimization procedure before convergence and generate then a non-optimal basis. We generated two non-optimal bases which correspond to different moments of the optimization procedure. We report the equivariance error and the classification error on the STL-10 dataset for DISCO with such bases functions in \Tabref{tab:equi_error}. It can be seen that lower equivariance errors correspond to lower classification errors.

\begin{table}[t]
    \begin{center}
    \begin{tabular}{|l|cccc|c|}
    \hline
    Model & SiamFC \cite{bertinetto2016fully} & TriSiam \cite{dong2018triplet} & SiamFC+\cite{zhang2019deeper} & SE-SiamFC+ \cite{Sosnovik_2021_WACV} & DISCO \\
    \hline\hline
    FPS & - & - & 56 & 14 & 28 \\
    AUC & $0.61$& $0.62$ & $0.67$ & $\textbf{0.68}$ & $\textbf{0.68}$ \\
    \hline
    \end{tabular}
    \end{center}
    \caption{Performance comparisons on the OTB-13 tracking benchmark. The best results are \textbf{bold}. We report the average number of framer per second (FPS) per sequence. Higher FPS and AUC are better.}
    \label{tab:otb_results}
\end{table}

\subsection{Tracking}
To test the ability of DISCO to deliver accurate scale estimation, we choose the task of visual object tracking. We take the recent SE-SiamFC+ \cite{zhang2019deeper} tracker and follow the recipe provided in \cite{Sosnovik_2021_WACV} to make it scale-equivariant. We employ the standard one-pass evaluation protocol to compare our method with conventional Siamese trackers and SE-SiamFC+ \cite{Sosnovik_2021_WACV} with a Hermite basis for the scale convolutions. The trackers are evaluated by the usual area-under-the-success-curve (AUC).

The scale-equivariant tracker with DISCO matches the performance of the state-of-the-art SE-SiamFC+, but twice faster as can be seen in Table \ref{tab:otb_results}. FPS is measured on Nvidia GTX 1080 Ti for all models. 

\subsection{Scene Geometry by Contrasting Scales}

We demonstrate the ability of DISCO to propagate scale information through the layers of the network, by presenting a simple approach for geometry estimation of a scene through the use of the intrinsic scale. This is possible because in the DISCO model, we can use high granularity of scale factors and process them more accurately and faster compared to other scale-equivariant models.

We construct a scale-equivariant network with DISCO layers. The weights are initialized from an ImageNet-pretrained network \cite{deng2009imagenet} following the approach described in \cite{Sosnovik_2021_WACV}. Next, we strip the classification head of the network and apply global spatial average-pooling. The resulting feature map thus has a dimension $B \times C \times S \times 1 \times 1$, where $B, C, S$ are the batch, channel and scale dimensions respectively. To decode the scale information, we sample the argmax along the scale dimension. Such a tensor has shape $B \times C$ where each element is a scalar that encodes the argmax for each of the objects on each of the channels. Then the tensor is passed to a shallow network, which produces a scale estimate for the input image. The feature extraction network followed by the shallow scale estimator network is denoted as $F_{\theta}$, where $\theta$ is the parameters of the shallow scale estimator, so we do not train the parameters of the feature extractor.

At the core of the method is the scale-contrastive learning algorithm. The model is trained to predict how much one image should be interpolated to match the other. Such an approach does not require any dedicated depth or scale labels. The algorithm is illustrated in Figure \ref{fig:contrast}. First, we sample randomly two scale factors $\gamma_1, \gamma_2 \sim U[0.5, 2.0]$ and apply interpolations $L_{\gamma_1}, L_{\gamma_2}$ to the image $\mathcal{I}$. The transformed images are fed into the network $F_{\theta}$, which predicts scale estimates $\tilde{\gamma}_1, \tilde{\gamma}_2$ (\Figref{fig:contrast}). Then, we minimize the following loss by using the Adam optimizer:
\begin{figure}[t]
  \centering
    \includegraphics[width=\linewidth]{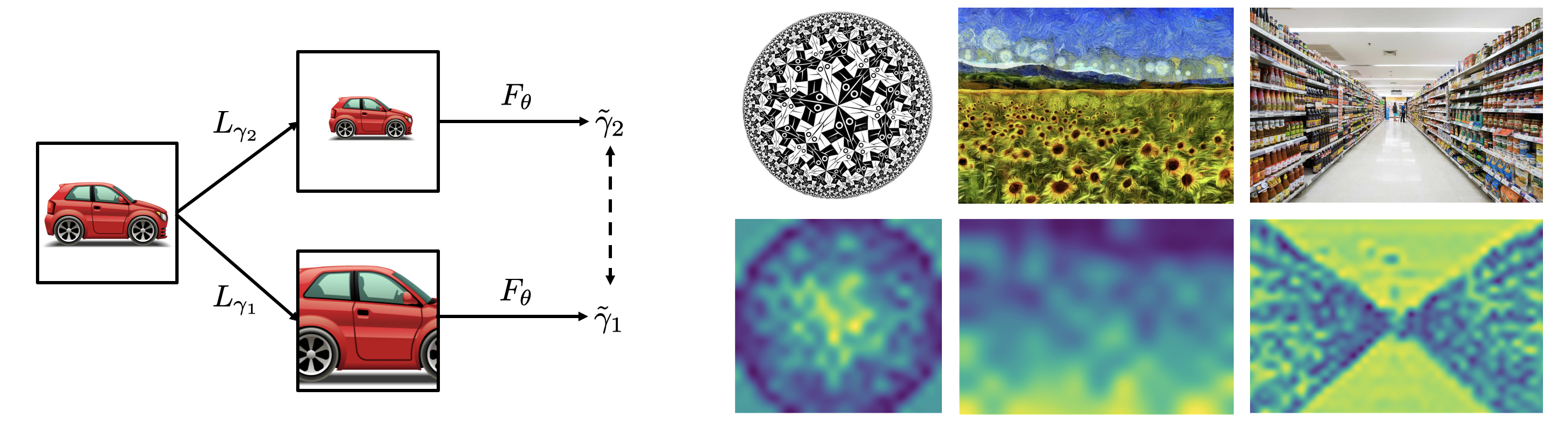}
  \caption{Left: the network is trained to predict the scale difference between an object and its resized version. Right: images and their scale fields produced by the DISCO model trained to contrast scales.}
  \label{fig:contrast}
\end{figure}

\begin{equation}
    \label{eq:geom}
    \mathcal{L}_\text{scale} = \mathbb{E}_{\mathcal{I}} 
    \Big[\frac{\gamma_2}{\gamma_1} - \frac{\tilde{\gamma}_2}{\tilde{\gamma}_1}\Big]^2 
    = \mathbb{E}_{\mathcal{I}} 
    \Big[\frac{\gamma_2}{\gamma_1} - \frac{F_{\theta}(L_{\gamma_2}(\mathcal{I}))}{F_{\theta}(L_{\gamma_1}(\mathcal{I}))}\Big]^2 
    \longrightarrow \min_\theta
\end{equation}

We train the model on the STL-10 dataset \cite{coates2011analysis} and evaluate it on random images found on the Internet. To infer the scene geometry of the image, we split the image into overlapping patches. For each of them we predict the scale. We provide qualitative results in Figure \ref{fig:contrast}. While the proposed methods was never trained on whole images, it captures the global geometry of the scenes, be it a road or a supermarket. 

We provide more detailed information for each of the experiments in supplementary materials. 

\section{Discussion}
\label{sec:discussion}

In this work, we demonstrate that the equivariance error affects the performance of equivariant networks. We introduce DISCO, a new class of kernels for scale-convolution, so the equivariance error is minimized. We develop a theory to derive an optimal rescaling to be used in DISCO and analyze under what conditions an optimal rescaling is possible and how to find a good approximation if these conditions do not hold. We also demonstrate how to efficiently incorporate DISCO into an existing scale-equivariant network.

We experimentally demonstrate that DISCO scale-equivariant networks outperform conventional and other scale-equivariant models, setting the new state-of-the-art on the MNIST-Scale and STL-10 datasets. In the visual object tracking experiment, DISCO matches the state-of-the-art performance of SE-SiamFC+ on OTB-13, however, works 2 times faster.

We suppose that the DISCO would be the most useful in problems, where an accurate scale analysis is required, such as multi-object tracking for autonomous vehicles, where the scale of objects can rapidly change due to the relative motion. We additionally want to highlight that the approach presented in this paper can be used to construct scale-equivariant self-attention models with reduced complexity \cite{romero2020group}.

\bibliography{egbib}

\appendix
\clearpage
\section{Proofs}

 We have shown that scale-convolution is indeed scale-equivariant only if the kernel $\kappa$ and its up-scaled version $\kappa_{s^{-1}}$ satisfy the following relation
 \begin{equation}
    \label{eq:supp_equi_constraint_func}
    L_s[f] \star \kappa = L_s[f \star \kappa_{s^{-1}}], \quad \forall f, s
\end{equation}
where $L_s$ is an operator of downscaling.

\subsection{Solutions in 1D}

Let us consider an operator of downscaling $L_s$, which is represented as a rectangular interpolation matrix $\mL$ of size $N_\text{out} \times N_\text{in}$. A convolution with a kernel $\kappa$ is represented as a multiplication with a matrix $\mK$ of size $N_\text{out}\times N_\text{out}$, and with a kernel $\kappa_{s^{-1}}$ written as a matrix $\mK_{s^{-1}}$ of size $N_\text{in}\times N_\text{in}$. The equivariance constraint with respect to $\kappa_{s^{-1}}$ is written as follows:

\begin{equation}
    \label{eq:supp_equi_constraint}
    \mK \mL = \mL \mK_{s^{-1}}
\end{equation}

\begin{lemma}
\Eqref{eq:supp_equi_constraint} has non-trivial solutions with respect to $\mK_{s^{-1}}$ only if $L$ performs downscaling by an integer factor.
\end{lemma}

\begin{proof}
Let us consider $\mP_{\text{in}}$ and $\mP_{\text{out}}$, matrices of circular shift of rows of sizes $N_{\text{in}}\times N_{\text{in}}$ and $N_{\text{out}}\times N_{\text{out}}$ correspondingly. With no loss of generality we assume circular boundary conditions for convolutions. Thus, matrices $\mK, \mK_{s^{-1}}$ are circulant, and therefore $\mK = \mP_{\text{out}} \mK \mP_{\text{out}}^T$ and $\mK_{s^{-1}} = \mP_{\text{in}} \mK_{s^{-1}} \mP_{\text{in}}^T$ \cite{loehr2014advanced}. If we substitute it into \Eqref{eq:supp_equi_constraint} we have the following:
\begin{equation}
    \label{eq:supp_equi_proof}
    \mP_{\text{out}}^{i} \mK (\mP_{\text{out}}^T)^{i} \mL
    = \mL \mP_{\text{in}}^j \mK_{s^{-1}} (\mP_{\text{in}}^T)^j, \quad \forall i, j \in \mathbb{Z}
\end{equation}
If we multiply it from the left by $(\mP_{\text{out}}^T)^{i}$ and from the right by $\mP_{\text{in}}^{j}$ we get the following equation:
\begin{equation}
    \label{eq:supp_equi_proof_2}
    \mK (\mP_{\text{out}}^T)^{i} \mL \mP_{\text{in}}^{j}
    = (\mP_{\text{out}}^T)^{i}\mL \mP_{\text{in}}^j \mK_{s^{-1}}
\end{equation}
We can now multiply \Eqref{eq:supp_equi_proof_2} by a coefficient $\alpha_{ij}$ and then the following holds true:
\begin{equation}
    \label{eq:supp_equi_proof_3}
    \mK 
    \sum_{i=1}^{N_{\text{out}}} 
    \sum_{j=1}^{N_{\text{in}}} 
    \alpha_{ij}\mQ_{ij}
    =     
    \sum_{i=1}^{N_{\text{out}}} 
    \sum_{j=1}^{N_{\text{in}}} 
    \alpha_{ij}\mQ_{ij} \mK_{s^{-1}}, \forall \alpha_{ij}
\end{equation}
where $\mQ_{ij}=(\mP_{\text{out}}^T)^{i} \mL \mP_{\text{in}}^{j}$. \Eqref{eq:supp_equi_proof_3} holds true for all $\alpha_{ij}$. Which gives us the following system of equations:
\begin{equation}
    \label{eq:supp_system}
    \begin{cases}
    \mK (\mQ_{00}-\mQ_{N_{\text{out}}N_{\text{in}}})
    = (\mQ_{00}-\mQ_{N_{\text{out}}N_{\text{in}}}) \mK_{s^{-1}} \\
    \mK (\mQ_{01}-\mQ_{N_{\text{out}},N_{\text{in}}+1})
    = (\mQ_{01}-\mQ_{N_{\text{out}},N_{\text{in}}+1}) \mK_{s^{-1}} \\
    \dots \\
    \mK (\mQ_{10}-\mQ_{N_{\text{out}}+1,N_{\text{in}}})
    = (\mQ_{10}-\mQ_{N_{\text{out}}+1,N_{\text{in}}}) \mK_{s^{-1}} \\
    \dots \\
    \end{cases}
\end{equation}
Which has non-trivial solutions if the expressions in all brackets are equal to zero. Thus, $\mQ_{00} = \mQ_{N_{\text{out}}N_{\text{in}}}$. In other words, $\mL$ is a row-circulant matrix and $N_{\text{in}}$ is divisible by $N_{\text{out}}$. Therefore, the downscaling is performed by an integer factor $N_{\text{in}}/N_{\text{out}}$
\end{proof}

In order to obtain the solution of \Eqref{eq:supp_equi_constraint} we represent convolutional matrices by using their eigendecompositions. 

\begin{equation}
\label{eq:supp_fdf}
    \begin{split}
&\mK = \mF_{\text{out}} \text{diag}(\mF_{\text{out}} \bm{\kappa}) \mF^*_{\text{out}} \\
&\mK_{s^{-1}} = \mF_{\text{in}} \text{diag}(\mF_{\text{in}} \bm{\kappa}_{s^{-1}}) \mF_{\text{in}}^*
    \end{split}
\end{equation}
where $\mF_{\text{in}},\mF_{\text{out}}$ are matrices of the Discrete Fourier Transform of appropriate sizes and $\bm{\kappa}, \bm{\kappa}_{s^{-1}}$ are vector representations of convolutional kernels. After substituting the second part of \Eqref{eq:supp_fdf} into \Eqref{eq:supp_equi_constraint} we obtain:
\begin{equation}
\label{eq:supp_kl_lk_decomposed}
\mK \mL = \mL \mF_{\text{in}} \text{diag}(\mF_{\text{in}} \bm{\kappa}_{s^{-1}}) \mF_{\text{in}}^*
\end{equation}
We then multiply both sides of the equation with $\mF_{\text{in}}$ from the right.
\begin{equation}
\label{eq:supp_kl_lk_decomposed2}
(\mK \mL \mF_{\text{in}})_{ij} = \sum_{k}(\mL \mF_{\text{in}} )_{ik}\text{diag}(\mF_{\text{in}} \bm{\kappa}_{s^{-1}})_{kj} 
\end{equation}
As the left hand side is per-column proportional to $\mL \mF_{\text{in}}$, we can calculate the solution just by using the first row of each matrix. 
\begin{equation}
\label{eq:supp_solution}
(\mF_{\text{in}} \bm{\kappa}_{s^{-1}})_j
= \frac{(\mK \mL \mF_{\text{in}})_{1j}}{(\mL \mF_{\text{in}})_{1j}}
\end{equation}
The first row of $\mF_{\text{in}}$ consists of ones so as the first row of $\mL \mF_{\text{in}}$. Additionally, $(\mK \mL \mF_{\text{in}})_{1j} = s^{-1}[\bm{\kappa}, \bm{\kappa}]_j$. As the discrete Fourier image of the solution is a scaled concatenated image of the source, the solution is just a dilation of the original kernel \cite{loehr2014advanced}.

\subsection{Solutions in 2D}
We are interested in solving \Eqref{eq:supp_equi_constraint_func} with respect to $\kappa_{s^{-1}}$ for any set of $\kappa$'s which forms a complete basis in the space of square matrices of a certain, fixed size. If the solution exists for any basis, then it exists for a basis of 2-dimensional separable kernels. As the rank of the set of solutions is less or equal to the rank of the initial basis, the solution is separable as well. Let us consider an image $\mF$ of size $N_\text{in} \times N_\text{in}$. Taking into account that its rescaling is a separable operation, the matrix form of \Eqref{eq:supp_equi_constraint_func} is:
\begin{equation}
    \label{eq:equi_constraint_2d}
    \mK' \mL \mF \mL^T \mK^T =   \mL \mK'_{s^{-1}}\mF \mK_{s^{-1}}^T \mL^T, \quad \forall \mF
\end{equation}
where $\mK'$ and $\mK$ are matrix representations of 1-dimensional components of a separable kernel. As \Eqref{eq:equi_constraint_2d} holds true for all images, it satisfies $\mF = \vf \vc^T$ and $\mF = \vc \vf^T$ where $\vc$ is a vector of constants and $\vf$ is an arbitrary vector. After substituting these functions into \Eqref{eq:equi_constraint_2d} it degenerates into a system of two independent equations up to a multiplication constant:
\begin{equation}
    \label{eq:equi_constrain_sep}
    \begin{cases}
     \mK \mL = \mL \mK_{s^{-1}} \\
     \mK' \mL = \mL \mK'_{s^{-1}}
    \end{cases}
\end{equation}
Thus, if a solution exists for 2-dimensional discrete signals it also exists for the 1-dimensional case.

\section{Implementation Details}

\subsection{Scale Convolution}

Let us consider a scale-convolutional layer defined on scales $\{1, \sqrt{2}, 2, 2\sqrt{2}, 4, \}$. The kernel on the smallest scale is of size $3 \times 3$. As it was noted, as soon as the kernel on the intermediate $\sqrt{2}$ scale is defined, all other kernel can be calculated via dilation. 

In scale-convolutional layer the kernels $\kappa$ are parametrized as follows:
\begin{equation}
    \label{eq:impl_kernel}
    \kappa_s = \sum_{j} \psi_{s,j} w_j
\end{equation}
where $\psi_{s,j}$ is a $j$-th basis function defined on scale $s$, and $w_j$ is the corresponding trainable coefficient. 

As the basis is fixed during the training, it needs to be defined \textit{a priori}. On the smallest scale all basis functions are just elements of the standard basis, i.e. if $\psi_{1,i}$ is the i-th basis function for the $3 \times 3$ filters on the first scale, then $\psi_{1,0}$ is a $3\times3$ matrix where the only non-zero element is a $1$ in the top-left corner, and $\psi_{1,4}$ is a $3\times3$ matrix with $1$ in the center. On the next integer scale $2$, the basis is obtained according to Equation 12 of the main paper and computed as a dilation of $\psi_{1, i}$. To obtain non-integer scale bases we start by approximating the first intermediate $\sqrt{2}$ scale basis $\psi_{\sqrt{2},j}$ functions by minimizing the following objective function:

\begin{equation}
\label{eq:impl_trainnable}
\|L[f] \star \psi_{1,j} - L[f \star \psi_{\sqrt{2},j}]\|_F^2
+ \|L[f] \star \psi_{\sqrt{2},j} - L[f \star \psi_{2,j}]\|_F^2
\end{equation}
where $f$ is a random sample from $\mathcal{N}(0, 1)$ and $L$ is an operation of downsampling by a factor of $\sqrt{2}$ by using bicubic interpolation. The basis the scale $\{2\sqrt{2} \}$ is calculated as a dilation of the approximated $\sqrt{2}$ basis. See Figure \ref{fig:scene_geom_supp} for more details.

After all basis functions are calculated, the basis is packed into a tensor of size: $$\texttt{num\_functions} \times \texttt{num\_scales} \times \texttt{height} \times \texttt{width}$$
and used for runtime kernel calculations with the algorithm provided by \cite{sosnovik2019scale}.

\begin{figure*}[t!]
  \centering
    \includegraphics[width=\linewidth]{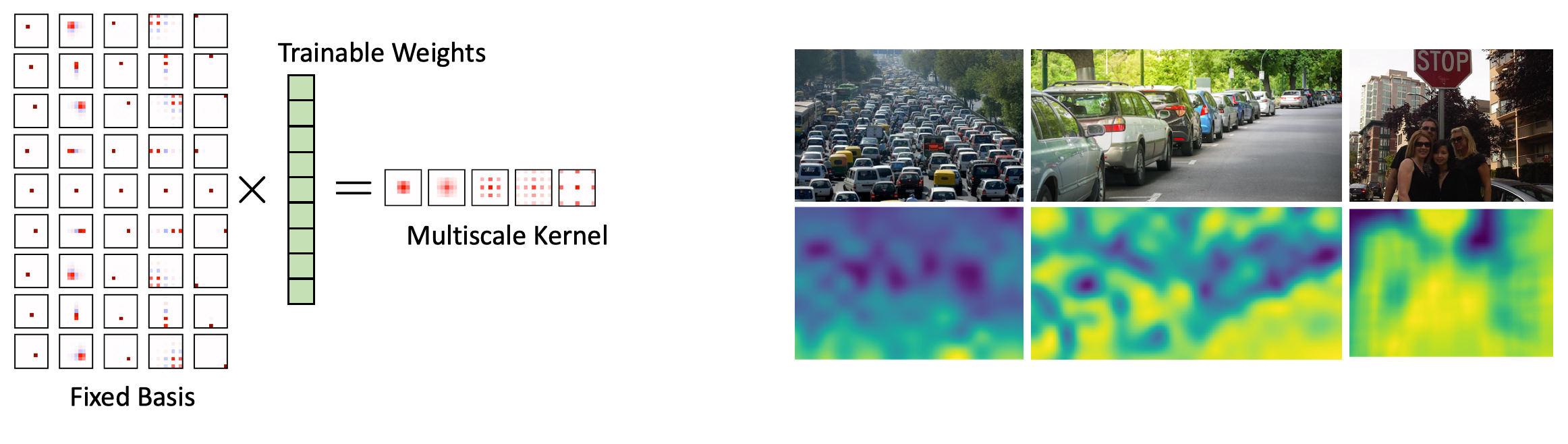}
  \caption{Left: kernels are computed via multiplying a fixed multi-scale basis with trainable weights. Right: images and their scale fields produced by the DISCO model trained to contrast scales.}
  \label{fig:scene_geom_supp}
\end{figure*}

\subsection{Computational Complexity}
Let us consider a scale-convolutional layer with a set of $N_s$ scales with step $\sigma > 1$. The smallest kernel size is $W \times W$. The computational complexity for calculating the output for one spatial position for the state-of-the-art method from \cite{sosnovik2019scale} can be estimated as follows:

\begin{equation}
    \label{eq:big_o_sesn}
        O(\textit{SESN}) 
        \sim O(W^2(1 + \sigma^2 + \dots + \sigma^{2N_s-2})) 
        \sim O\Big(W^2 \frac{\sigma^{2N_s} - 1}{\sigma^{2} - 1}\Big) 
        \sim O(W^2 \sigma^{2N_s}) 
\end{equation}

In contrast, for DISCO we arrive the following complexity:
\begin{equation}
    \label{eq:big_o_disco}
    \begin{split}
        O(\textit{DISCO}) \sim O(N_s W^2)
    \end{split}
\end{equation}
Thus, where the state of the art SESN convolution grow exponentially in computational complexity with the number of scales, DISCO allow for linear growth. 

When using a scale step of $\sqrt{2}$ we achieve a speedup of:
\begin{equation}
    \label{eq:speepdup}
    \frac{O(SESN)}{O(DISCO)} \sim \frac{2^{N_s}}{N_s}.
\end{equation}
The main reason for the acceleration is that in SESN the filters are dense, as they are rescaled in the continuous domain by using Equation 4 of the main paper, while DISCO filters are sparse as the rescaling is performed by using dilation for the majority of scales. The actual speedup depends on the particular implementation of scale-convolution with such kernels. The current implementation is limited by the functionality of modern deep learning software which is not optimized for sparse filters of a big spatial extend.

\subsection{General Solution}
While in many models which consider scale the scale-step is a root of some integer number, it is possible to build a DISCO model with arbitrary scale-steps. 
Let us consider a scale-convolutional layer defined on scales $\{s_0, as_0, a^2s_0, \dots a^Ns_0 \}$ where $a>1$. In order to construct kernels for such a layer it is first required to calculate a basis $\{\psi_{s_0, j}, \psi_{as_0, j}, \dots \psi_{a^Ns_0, j} \}$ for all $j$. The basis can be calculated as a minimizer of the following objective:
\begin{equation}
        \mathcal{L}(\psi_{s_0, j}, \psi_{as_0, j}, \dots \psi_{a^N s_0, j} )
        = \mathbb{E}_f \sum\limits_{\substack{k,l=0 \\ k > l}}^{k,l=N} 
        \|L_{a^{l-k}}[f]\star \psi_{a^{l}s_0,j} - L_{a^{l-k}}[f \star \psi_{a^{k}s_0,j} ]\|_F^2
\end{equation}

\section{Experiments}

\begin{table}[t]
    \begin{center}
    \begin{tabular}{|l|ccc|}
    \hline\hline
    Interpolation & Nearest & Bilinear & Bicubic \\
    \hline
    Error & $1.36\pm0.06$ & $1.37\pm0.05$ & $1.35\pm0.05$ \\ 
    \hline
    \end{tabular}
    \end{center}
    \caption{Classification accuracy on MNIST-scale for different interpolation methods which are used for approximate basis calculations.}
    \label{tab:interpolation}
\end{table}

\subsection{MNIST-scale}
As a baseline model we use the SESN model \cite{sosnovik2019scale}. It consists of 3 convolutional and 2 fully-connected layers. Each layer has filters of size $7\times7$. We keep the number of parameters the same for all SESN models and for DISCO. The main difference between the SESN and DISCO models is in the basis for scale-convolutions. We also discovered that average-pooling works slightly better for the DISCO, while for all other methods it either has no effect or worsens the performance. Both SESN and DISCO use the same set of scales in scale convolutions: $\{1, 2^{1/3}, 2^{2/3}, 2\}$

All models are trained with the Adam optimizer \cite{kingma2014adam} for 60 epochs with a batch size of 128. We set the initial learning rate at $0.01$ and divide it by 10 after 20 and once more after 40 epochs. We conduct the experiments with 2 different settings: without data augmentation and with scaling augmentation. We run the experiments on 6 different realizations of the MNIST-scale. We report the mean $\pm$ standard deviation over these runs.

We found in our experiments that the interpolation method which is used to calculate a basis by using \eqref{eq:impl_trainnable} does not affect the final solution. The relative mean squared error between bases is less than percent. Moreover, DISCO model demonstrates almost the same results on MNIST-scale while various interpolation methods are used. See Table \ref{tab:interpolation} for more results.

\subsection{STL-10}
As a baseline we use WideResNet \cite{Zagoruyko2016WRN} with 16 layers and a widening factor of 8. Scale-equivariant models are constructed according to \cite{sosnovik2019scale}. All models have the same number of parameters. The scale factors in the scale convolutions are $\{1, \sqrt{2}, 2\}$.

The models are trained for 1000 epochs using the SGD optimizer with a Nesterov momentum of $0.9$ and a weight decay of $5\cdot10^{-4}$. For DISCO, we increase the weight decay to $1\cdot10^{-4}$. Tuning weight decay for the other models did not bring any improvement. The learning rate is set to $0.1$ at the start and decreased by a factor of $0.2$ after the epochs 300, 400, 600 and 800. The batch size is set to 128. During training, we additionally augment the dataset with random crops, horizontal flips and cutout \cite{devries2017cutout}.

\subsection{Scene Geometry by Contrasting Scales}

For clarity we provide a PyTorch pseudo-code for DISCO scene geometry estimation (Listing \ref{listing}). We utilize scale-equivariant ResNet as a backbone feature extractor. The produced feature map is reduced in a spatial domain. Then the \texttt{argmax} along the scale dimension is extracted and passed to the scale MLP regressor to produce a scale estimate. Additional qualitative results are presented in Figure \ref{fig:scene_geom_supp}.

\begin{lstlisting}[language=Python, label=listing, caption=PyTorch pseudo-code for DISCO scene geometry estimation.]
import torch.nn as nn 
import SE_ResNet

class ScaleEstimator(nn.Module):
    def __init__(self):
        super().__init__()
        self.backbone = SE_ResNet(pretrained=True)
        self.regressor = nn.Sequetial(
            nn.Linear(512, 256), 
            nn.ReLU(),
            nn.Linear(256, 1),
            nn.ReLU()
        )

    def forward(self, x):
        # x.shape = B, 3, 64, 64
        y = self.backbone(x)
        # y.shape = B, 512, 9, 1, 1
        y = y.mean(-1).mean(-1)
        # y.shape = B, 512, 9
        y = y.argmax(-1)
        scale = self.regressor(y)
        return scale
        
\end{lstlisting}

\end{document}